\documentclass[journal]{IEEEtran}
\usepackage{url} 
\usepackage[square,numbers,sort&compress]{natbib}
\usepackage{cite}
\usepackage{amsthm,amsmath,amssymb}
\usepackage{mathrsfs}
\usepackage{array}
\usepackage{dsfont}
\usepackage{tikz}
\usepackage[skins]{tcolorbox} 
\tcbuselibrary{breakable} 
\usepackage{booktabs}  
\usepackage{multirow}
\usepackage{hyperref}

\usepackage[capitalize]{cleveref}
\crefname{section}{Sec.}{Secs.}
\Crefname{section}{Section}{Sections}
\Crefname{figure}{Figure}{Figures}
\Crefname{figure}{Fig.}{Figs.}
\Crefname{table}{Table}{Tables}
\crefname{table}{Tab.}{Tabs.}
\crefname{algorithm}{Algorithm}{Algorithms}
\crefname{algorithm}{Alg.}{Algs.}

\newtheorem{definition}{Definition}[section]
\newtheorem{theorem}{Theorem}[section]

\newtheorem{remark}{Remark}[section]
\newtheorem{lemma}{Lemma}[section]

\usepackage[caption=false, font=normalsize, labelfont=sf, textfont=sf]{subfig}

\usepackage{etoolbox}
\makeatletter
\patchcmd{\@makecaption}
  {\scshape}
  {}
  {}
  {}
\makeatletter
\patchcmd{\@makecaption}
  {\\}
  {.\ }
  {}
  {}
\makeatother

\usepackage{flushend} 
\usepackage{tabularx} 
\usepackage{ragged2e} 
\newcolumntype{L}{>{\RaggedRight\hangafter=1\hangindent=0em}X}

\usepackage[ruled,linesnumbered]{algorithm2e}
\usepackage{courier}
\usepackage{xspace}
\makeatletter
\DeclareRobustCommand\onedot{\futurelet\@let@token\@onedot}
\def\@onedot{\ifx\@let@token.\else.\null\fi\xspace}

\def\eg{\emph{e.g}\onedot} 
\def\ie{\emph{i.e}\onedot} 
 
\def\etc{\emph{etc}\onedot} \def\vs{\emph{vs}\onedot}
 
\def\iid{i.i.d\onedot} 

\makeatother

\definecolor{blueryb}{rgb}{0.01, 0.28, 1.0}
\definecolor{aoen}{rgb}{0.0, 0.5, 0.0}
\definecolor{awesome}{rgb}{1.0, 0.13, 0.32}
\definecolor{americanrose}{rgb}{1.0, 0.01, 0.24}

\usepackage[normalem]{ulem}
\newcounter{ToDo}
\newcounter{gaocomm}

\newcounter{Note}
\definecolor{blue-violet}{rgb}{0.00,0.75,0.90}
\definecolor{mygreen}{rgb}{0.0, 0.5, 0.0}
\definecolor{awesome}{rgb}{1.0, 0.13, 0.32}
\definecolor{bostonuniversityred}{rgb}{0.8, 0.0, 0.0}

\begin{document}
\title{Mixture Data for Training Cannot Ensure Out-of-distribution Generalization}

\author{Songming~Zhang,
        Yuxiao~Luo,
        Qizhou~Wang,
        Haoang~Chi,\\
        Xiaofeng~Chen,
        Bo~Han, 
        Junbin Gao,
        and~Jinyan~Li
\thanks{\emph{Corresponding author: Jinyan~Li.}}
\thanks{The work was done at Shenzhen Institute of Advanced Technology, Chinese Academy of Sciences.
S. Zhang and X. Chen are with Department of Mathematics, Chongqing Jiaotong University, China (sm.zhang1@siat.ac.cn). Q. Wang and B. Han are with Hong Kong Baptist University. H. Chi is with National University of Defense Technology. J. Gao is with Discipline of Business Analytics, The University of Sydney Business School, The University of Sydney, Australia (junbin.gao@sydney.edu.au).  Y. Luo and J. Li are with Shenzhen Institute of Advanced Technology, Chinese Academy of Sciences (jinyan.li@siat.ac.cn)}
}

\markboth{Journal of \LaTeX\ Class Files}%
{Zhang \MakeLowercase{\textit{et al.}}: Bare Demo of IEEEtran.cls for IEEE Journals}

\maketitle

\begin{abstract}
 Deep neural networks often face generalization problems to handle out-of-distribution (OOD) data, and there remains a notable theoretical gap between the contributing factors and their respective impacts. Literature evidence from in-distribution data has suggested that generalization error can shrink if the size of mixture data for training increases. However, when it comes to OOD samples, this conventional understanding does not hold anymore---Increasing the size of training data does not always lead to a reduction in the test generalization error. In fact, diverse trends of the errors have been found across various shifting scenarios including those decreasing trends under a power-law pattern, initial declines followed by increases, or continuous stable patterns. Previous work has approached OOD data qualitatively, treating them merely as samples unseen during training, which are hard to explain the complicated non-monotonic trends. In this work, we quantitatively redefine OOD data as those situated outside the convex hull of mixed training data and establish novel generalization error bounds to comprehend the counterintuitive observations better. The new error bound provides a tighter upper bound for data residing within the convex hull compared to previous studies and is relatively relaxed for OOD data due to consideration of additional distributional differences. Our proof of the new risk bound agrees that the efficacy of well-trained models can be guaranteed for unseen data within the convex hull; More interestingly, but for OOD data beyond this coverage, the generalization cannot be ensured, which aligns with our observations. Furthermore, we attempted various OOD techniques (including data augmentation, pre-training, algorithm power, \etc) to underscore that our results not only explain insightful observations in recent OOD generalization work, such as the significance of diverse data and the sensitivity to unseen shifts of existing algorithms, but it also inspires a novel and effective data selection strategy.
  
\end{abstract}


%

\section{Introduction}

\begin{figure*}[!tb]
    \centering
    \includegraphics[width=.75\linewidth]{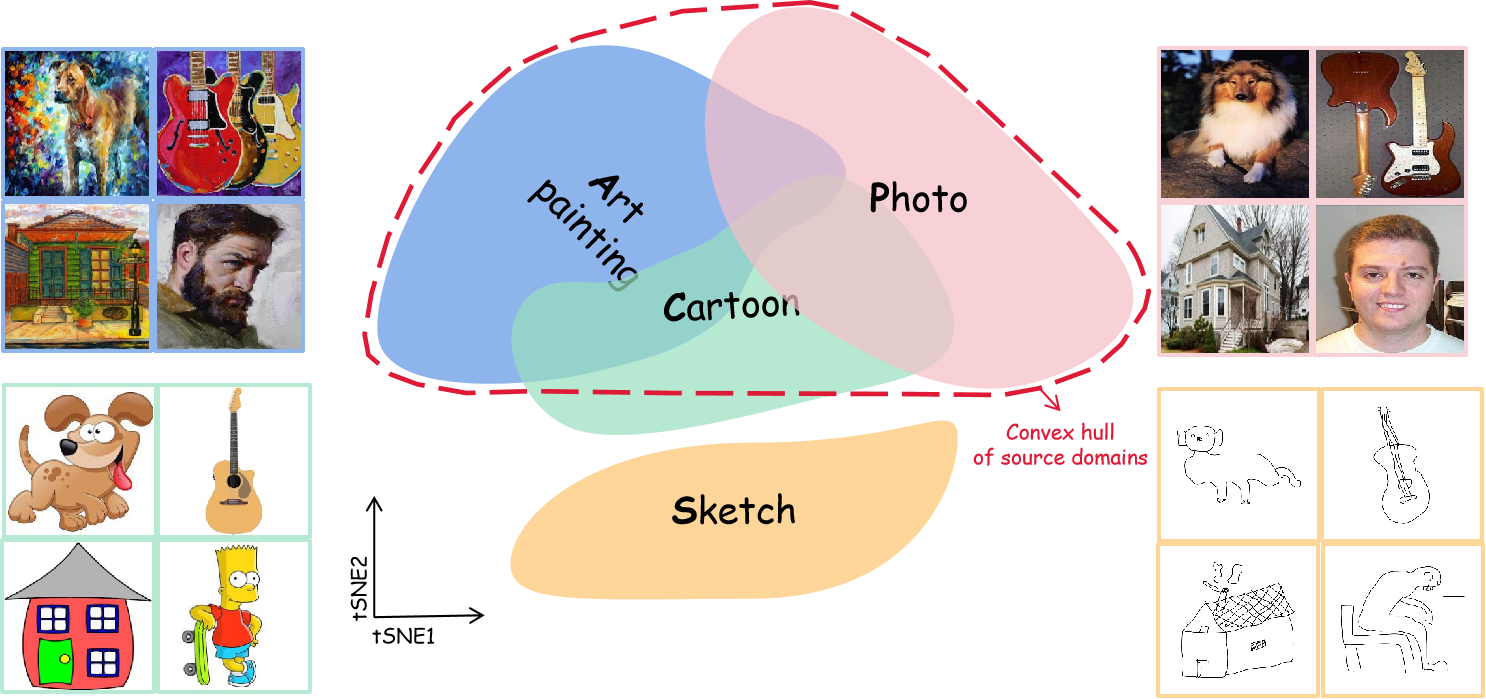}
    \caption{A schematic diagram of a multi-domain sample in practice which consists of source and target domains. Suppose we can only have access to \texttt{Painting} and \texttt{Photo}, the model exhibits different generalization abilities at different OOD domains \texttt{Cartoon} and \texttt{Sketch} according to the distance to the mixture of training data. We draw a counterintuitive conclusion that the efficacy of well-trained models cannot be guaranteed for OOD data beyond the convex hull of training mixture, which is consistent with our experimental observations in \Cref{sec:obser}.}
     \label{fig:eg}
\end{figure*}

\IEEEPARstart{R}{eal-world} data are often sourced from diverse domains, where each source is characterized by a different distribution shift, and unknown shifts are hidden in their test distributions (see \Cref{fig:eg})~\citep{zhang2021deep, yang2023change}.
While deep neural networks (DNNs) demonstrate proficiency with in-domain data, they have difficulties in gaining generalization on unknown data shifts.
Most out-of-distribution (OOD) generalization methods typically assume their model's capability to extrapolate across every unseen shift and have been working on algorithm improvements such as regularization~\citep{chen2022pareto}, robust optimization~\citep{azizian2024exact}, and adjustments in model architecture~\citep{li2023sparse}.
Despite these efforts, theoretical analysis of the key factors influencing unseen data and their impacts on model performance is still lacking.
Previous empirical evidence such as neural scaling law~\citep{kaplan2020scaling, Cherti2023reproduce} suggests that all generalization errors follow the same decreasing trend as a power of training set size. Thus, it means that the addition of training data can effectively minimize generalization errors even for unknown distributions.

However, it appears that the scaling law conclusion may not hold true in practice, potentially leading to counterintuitive outcomes.
In fact, the OOD generalization error can have a diverse range of scenarios, including decreases under a power-law pattern, initial decreases followed by increases, or remaining stable.
We thus propose that \emph{not all generalization errors in unseen target environments will decrease when training data size increases.}
As the generalization error decreases, the model's accuracy improves, indicating better generalization capability on OOD data.
In other words, we conjecture that simply increasing the volume of training data may not necessarily enhance generalization to OOD data.
To our knowledge, no theoretical or empirical methods in the literature have addressed this phenomenon.

To formally understand such problems, this paper first presents empirical evidence for non-decreasing error trends under various experimental settings on the MNIST, CIFAR-10, PACS, and DomainNet datasets.
The results are then used to illustrate that with the expansion of the training size, the generalization error decreases when the test shift is relatively minor,  mirroring the performance observed under in-distribution (ID) data.
Yet, when the degree of shift becomes substantial, the generalization error may not decrease monotonically.
Previous works often indiscriminately categorized data unseen during training as OOD data without acknowledging the underlying causes of the non-monotonic patterns that may contain.
This motivates us to revisit fundamental OOD generalization set-ups to elucidate such non-decreasing trends.

We propose a novel theoretical framework within the context of OOD generalization.
Given a set of training environments, we first argue that OOD data can be redefined as a type of data lying outside the convex hull of source domains.
Based on this new definition, we prove new bounds for OOD generalization errors.
Our revised generalization bound establishes a tighter upper bound for data within the convex hull, while being relatively relaxed for OOD data when considering additional distributional divergences.
Further analysis of this new bound yields some theoretical insights into the influence of shift degree and data size on generalization error, and even on model capacity.
By effectively distinguishing between ID and OOD cases in the unseen data, we demonstrate that model performance can be guaranteed for unseen yet ID data with extrapolation across training environments.
However, our results collectively highlight that \textbf{being trivially trained on data mixtures cannot guarantee the OOD generalization ability of the models}, \ie, the model cannot infinitely improve its OOD generalization ability by increasing training data size.
Thus, achieving comparable performance on OOD data remains a formidable challenge for the model.

With our new theoretical framework in place, the ensuing question pertains to enhancing its predictive efficacy in scenarios devoid of prior target knowledge and providing coherent explanations. First, we employ widely used techniques to assess model adaptability, including data augmentation, pre-training, and algorithms. Drawing from our new theoretical findings, the efficacy of these invaluable tools can be elucidated by their ability to expand the coverage of the training mixture and its associated convex hull. Pre-training enables the model to acquire broader and more generalized representations from pre-trained datasets, while data augmentation facilitates an increase in the diversity of representations by expanding the size of training data. In contrast, hyperparameter optimization yields poor results since it solely modifies the training hyperparameters without providing more insight regarding the OOD samples. Moreover, it is important to note that existing algorithms are also sensitive to unseen test shifts.

We also note that this work distinguishes itself not just by offering theoretical understandings for error scenarios on DNNs and common techniques used for OOD generalization, but also by presenting novel insights for algorithm design. Specifically, inspired by the analysis of data diversity in our new definition, we proceed to evaluate a novel data selection algorithm that relies only on training data.
By selectively choosing samples with substantial differences alone, the coverage of the training mixture can be effectively expanded, consequently broadening representations learned by the models.
Remarkably, this algorithm surpasses the baseline, particularly in the case of large training size, whether chosen randomly or using reinforcement learning techniques.
As our focus is on data preprocessing without the requirement of environmental labels, this allows for a smooth combination with other OOD generalization methods to improve further the model's capability to handle unseen shifts.

Overall, we have made three significant contributions in this study:
\begin{enumerate}
    \item Contrary to the widely held ``more data, better performance'' paradigm, we draw a counterintuitive picture: simply increasing training data cannot ensure model performance especially when distribution shifts occur in test data. Our empirical conclusion stems from the complicated non-decreasing trends of OOD generalization errors.
    \item We proposed a novel definition for OOD data and proved new error bounds for OOD generalization. Our further analysis of this new bound revealed the main factors with their influence leading to non-decreasing tendency and provided strong support for our empirical conclusions.
    \item We explored and validated popular techniques such as data augmentation, pre-training, and algorithm tricks, demonstrating that our new theoretical results not only explain their effectiveness but also guide a novel data selection method for superior performance.
\end{enumerate} 
\section{Generalization error of patterns observed from OOD scenarios}
\label{sec:obser}
\begin{table*}[!tb]
    \caption{Summary of network architectures used in the experiments.\\ We select different architectures for different tasks to better observe OOD generalization error patterns.}
    \centering
    \begin{sc}
    \begin{tabular}{llrrr}
    \toprule[1pt]
     \textbf{Experiment}&  \textbf{Network(s)}&  \textbf{\#classes}&  \textbf{Image Size}& \textbf{Batch Size} \\ \midrule
     Rotated CIFAR-10&  WRN-10-2&  2&  $(3,32,32)$& 128 \\
     Blurred CIFAR-10&  WRN-10-2&  2&  $(3,32,32)$& 128 \\
     Split-CIFAR10&  SmallConv, WRN-10-2&  2&  $(3,32,32)$& 128 \\
     Rotated MNIST&  SmallConv&  10&  $(1,28,28)$& 128 \\
     PACS&  WRN-16-4&  4&  $(3,64,64)$& 16 \\
     DomainNet&  WRN-16-4&  40&  $(3,64,64)$& 16 \\
     CINIC-10&  WRN-10-2&  10&  $(3,32,32)$& 128 \\
    \bottomrule[1pt]
    \end{tabular}
    \label{tab:net}
    \end{sc}
\end{table*}

To thoroughly assess the model's generalization ability, especially its performance in the presence of distribution shifts, we examine generalization patterns from different scenarios of OOD data. OOD generalization refers to the model's capacity to perform well on those test data containing a distribution distinct from those occurring in the training data~\citep{zhou2022domain}. Traditionally, it has been intuitively assumed that as the training size increases, the model's generalization performance improves and the generalization error decreases accordingly, which was also verified experimentally by neural scaling law~\citep{kaplan2020scaling}. However, this assumption needs to be verified by specific experiments to ensure whether it is still applicable to OOD data.

\subsection{Experimental Settings for OOD Evaluation}

To systematically evaluate the model's ability to generalize on OOD samples, we designed a series of experiments. Our objective through these experiments is to reveal the relationship between model's OOD capability and distribution shift and to explore whether and why the model can maintain its performance under unknown shifts.  The summary of the datasets, network architectures, and training details used in the experiments are presented as follows.

\subsubsection{Datasets containing OOD distributions}
The OOD sub-tasks follow the setting  outlined by~\citet{de2023value} and are constructed from the images available at CIFAR-10, CINIC-10, and several datasets at DomainBed~\citep{gul2020search} such as Rotated MINST~\citep{ghifary2015domain}, PACS~\citep{li2017deeper} and DomainNet~\citep{peng2019moment}. Two types of settings are attempted to examine the impact of training data size on OOD data.

\textbf{OOD data arising due to correlation shift.}
We investigate how correlation shift affects a classification task from a transformed version of the same distribution. We consider the following scenarios for this purpose.
\begin{enumerate}
    \item \textbf{Rotated CIFAR-10:}
    $0^\circ$ and $60^\circ$-rotated images as training distribution, and $\theta_1^\circ$-rotated images ($30^\circ - 150^\circ$) as unseen target.  This scenario tests how rotation changes the appearance of natural images.
    \item \textbf{Blurred CIFAR-10:}
    $0$ and $3$-blurred images as training distribution and $\sigma_1$-blurred images with a range of blurring levels from $2$ to $20$ as the unseen target.
    This scenario tests how blur changes the clarity of natural images. Here, ``blur'' refers to adding a corresponding degree of Gaussian blur to the original images.
    \item \textbf{Rotated MNIST:}
    $0^\circ$ and $30^\circ$-rotated digits as training distribution, and $\theta_1^\circ$-rotated digits ($15^\circ - 60^\circ$) as unseen target. This scenario tests how rotation changes the appearance of handwritten digits.
\end{enumerate}

\begin{figure*}[!tb]
    \centering
    \includegraphics[width=\linewidth]{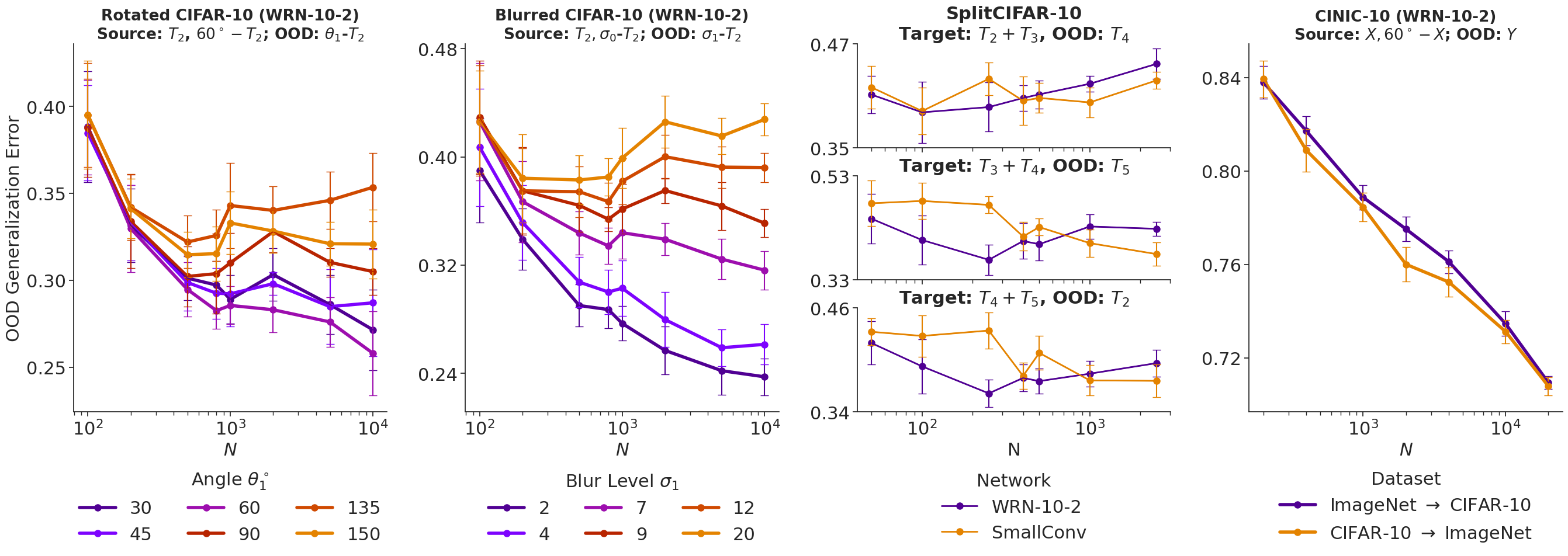}
    \caption{
    The lower the OOD generalization error, the better the model is at handling unseen targets. Error bars indicate $95\%$ confidence intervals (10 runs).
    \textbf{(a)} Different angles $\theta_1$ as unseen samples obtained by rotating images in OOD sub-task T2 (\texttt{Bird} \emph{vs.} \texttt{Cat}) in CIFAR-10, with  $0^{\circ}$ and $60^{\circ}$ as training samples, $M=400$.  For small $\theta_1$, increasing training data size improves the OOD generalization ability of the model. However, beyond a certain value of $\Delta_1$, the error with large rotation has a non-monotonic trend, which means overfitting on unseen rotation. 
    \textbf{(b)} $2-20$ level of Gaussian blur are unseen samples, and the training blur levels are at 0 and 3, $M=400$. The model is resilient to unobserved blur, yet for extreme levels of blur, non-monotonic scenarios are evident, indicating that the model is misaligned with data due to noise. 
    \textbf{(c)} Generalization error of two separate networks, WRN-10-2 and SmallConv, concerning a given unseen task. Our plots involve 3 different task pairs from Split-CIFAR10 and exhibit the generalization error as a function of the number of training samples. All 3 pairs demonstrated a non-decreasing trend in OOD generalization for both network models.
    \textbf{(d)} Generalization error of two separate datasets in CINIC-10, consisting of CIFAR-10 and ImageNet subsuet. We set one as the training environment and the other as OOD. While the purple curve shows higher error due to distribution shift, we did not observe any non-monotonic trend when testing on the unseen samples. Even when transferring between different datasets, the degree of distribution shift is still the main factor.
    }
    \label{fig:syn_1}
\end{figure*}

\textbf{OOD data caused by diversity shift.}
We also study how diversity drifts affect a classification task using data samples from source distributions and OOD samples from a different distribution. Diversity shift is a change in diversity or variability of the data distribution between training data and OOD data~\citep{ye2022ood}.
We consider the following scenarios for this purpose.
\begin{enumerate}
    \item \textbf{CINIC-10:}
    The construction of the dataset motivates us to consider two sub-tasks from CINIC-10: (1) Distribution from CIFAR images to ImageNet, and (2) Distribution from ImageNet images to CIFAR. This scenario tests how well a model can recognize images from another data distribution.
    \item \textbf{Split-CIFAR10:}
    We use two of the five binary classification tasks from Split-CIFAR10 as training data and another as the unseen task. This scenario tests how well a model can distinguish between different categories of natural images.
    \item \textbf{PACS:}
    We use two of the four-way-classification from four domains (Photo, Art Painting, Cartoon and Sketch) and images from one of the other unused domains as the unseen task.  This scenario tests how well a model can generalize to different styles and depictions of images.
    \item \textbf{DomainNet:}
    Same settings as PACS, we use a 40-way classification from 6 domains in DomainNet (\texttt{clipart}, \texttt{infograph}, \texttt{painting}, \texttt{quickdraw}, \texttt{real}, and \texttt{sketch}).
    This scenario tests how well a model can adapt to different domains of images with varying levels of complexity and diversity
\end{enumerate}

\subsubsection{Details for Training}
\label{sec:app_net}

For each random seed, we randomly select samples of different sizes $n_0$ and $n_1$ from the source distribution, making their total number as $N=n_0+n_1$. Next, we select OOD samples of a fixed size $M$ from the unseen target distribution. For Rotated MNIST, Rotated CIFAR-10, and Blurred CIFAR-10, the unseen distribution refers to never appearing at the rotated or blurred level. For PACS and DomainNet, images are down-sampled to $(3, 64, 64)$ during the training.

\subsubsection{Neural Architectures}
In our experiments, we utilize three different network architectures:
(a) a small convolutional network with $0.12$M parameters (\emph{SmallConv}),
(b) a wide residual network of depth $10$ and a widening factor $2$ (\emph{WRN-10-2}), and
(c) a larger wide residual network of depth $16$ and widening factor $4$ (\emph{WRN-16-4})~\citep{zagoruyko2016wide}.
SmallConv consists of $3$ convolution layers with a kernel size of $3$ and $80$ filters, interweaved with max-pooling, ReLU, batch-norm layers, and a fully-connected classifier layer.

\Cref{tab:net} provides a summary of the network architectures used.
All networks are trained using stochastic gradient descent with Nesterov's momentum and cosine-annealed learning rate scheduler.
The training hyperparameters include a learning rate of $0.01$ and a weight-decay of $10^{-5}$.

\subsection{Generalization error scenarios for deep learning benchmark datasets}
\label{sec:obser_real}

\textbf{Not all generalization errors decrease due to correlation shift.}
The shift in correlation refers to the change in the statistical correlations between the source and unseen target distribution~\citep{ye2022ood}.
Five binary classification sub-tasks use CIFAR-10 to explore the generalization scenario of unseen data. Our research focuses on a CIFAR-10 sub-task $T_2$ (\texttt{Bird} \vs \texttt{Cat}), where we use rotated images with $0^{\circ}$ and $60^{\circ}$ as training environment. Also, we use the rotated images with fixed angles from $30^{\circ} - 150^{\circ}$ as OOD samples. We also investigate the OOD effect of applying Gaussian blur with different levels to sub-task $T_2$ from the same distribution. Our results in \Cref{fig:syn_1}(a)-(b) both show a monotonically decreasing trend within the generalization error for the low-level shift, \ie, small rotation, and low blur.  However, for a high-level shift, it is a non-monotonical function of training sample size on the target domain. Despite enlarged training data, the generalization error remains relatively high. This can be explained in terms of overfitting training data: the model learns specific patterns and noise of the training data (such as rotation or blur) but fails to capture the underlying representations of the data. This compromises the model’s robustness to variations in distribution and its generalization to new data.

\textbf{Not always a decreasing trend can occur when OOD samples are drawn from a different distribution.}
OOD data may arise due to categories evolving with new appearances over time or drifting in underlying concepts (for instance, an airplane in 2024 with a new shape or even images of drones)~\citep{zenke2017continual}. Split CIFAR-10 with 5 binary sub-tasks is used to study the generalization scenario of unseen data, such as \texttt{frog} \vs \texttt{horse} and \texttt{ship} \vs \texttt{truck}. We consider sub-task combinations $(T_i, T_j)$ as the training domain and evaluate the trend of error on $T_k$. As the sample size grows, see  \Cref{fig:syn_1}c, the generalization error shrinks slightly or even does not show a falling trend, either in WRN-10-2 or SmallConv. Interestingly, WRN-10-2 initially outperforms SmallConv but it is overtaken by the latter as the number of samples increases.

\begin{figure*}[tb]
    \centering
    \includegraphics[width=\linewidth]{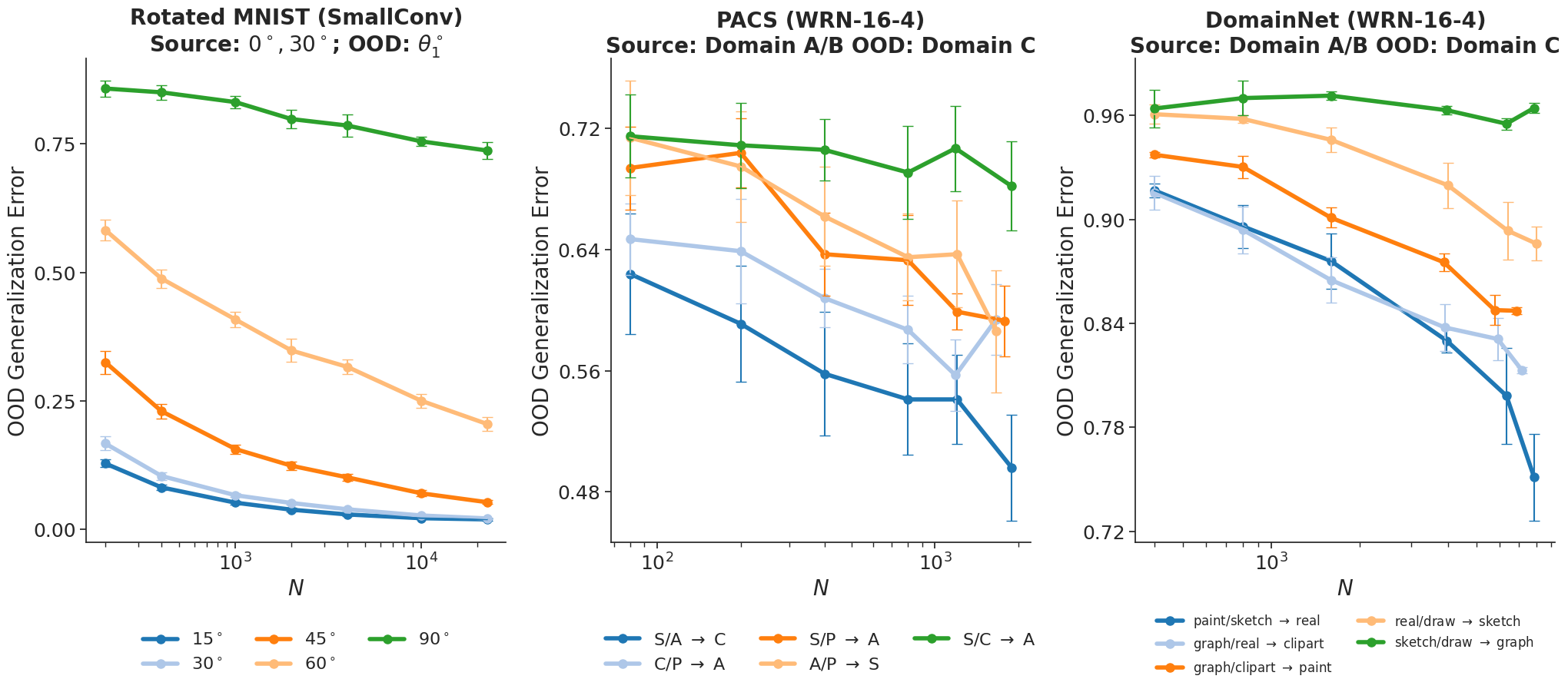}
    \caption{Different error trends in OOD generalization error on three DomainBed benchmarks.
    \textbf{Left:} Rotated MNIST (10 classes, $M=2,000$, SmallConv),
    \textbf{Middle:} PACS (4 classes, 4 domains \{A, C, P, S\}, $M=25$, WRN-16-4),
   \textbf{Right:} DomainNet (40 classes, 6 domains \{paint, sketch, real, graph, clipart, draw\}, $M=25$, WRN-16-4).
   Error bars indicate $95\%$ confidence intervals (10 runs for Rotated MNIST and PACS, 3 runs for DomainNet).
   As the number of training samples increases, the various distances between distributions and how they are combined lead to different decreasing trends in OOD generalization error.
    }
    \label{fig:syn_2}
\end{figure*}

\textbf{Non-decreasing trend also occurs for OOD benchmark datasets.}
The different trends of generalization error motivate us to further investigate three popular datasets in the OOD works~\citep{ye2022ood, arjovsky2019invariant}. First, we examine Rotated MNIST benchmark when the OOD samples are represented by $\theta$-rotated digit images in MNIST, while $0^\circ$ and $30^\circ $as training angles.
We observe a decreasing trend in \Cref{fig:syn_2} (left), but the error lower bound keeps rising and the slope keeps getting smaller as the testing angle increases. This shows that even in real-world datasets, model generalization is also vulnerable to unknown shifts.
Thus, we explore the PACS~\citep{li2017deeper} and DomainNet~\citep{peng2019moment} dataset from DomainBed benchmark~\citep{gul2020search}. The dataset contains subsets of different domains, two of which we selected as training domains and the other as an unseen target domain. When training samples consist of sketched and painted images,
the generalization error on the clipart domain falls exponentially (\ie, S/A $\rightarrow$ C in~\Cref{fig:syn_2} (Middle)).
Moreover, an interesting observation is that the generalization error tested on the graphical images drops only slightly and remains consistently high when learned from parts of drawings and sketches. Similar trends are also observed in DomainNet, which is a comparable benchmark to PACS;  See~\Cref{fig:syn_2} (Right).

\begin{figure}[!tb]
    \centering
    \includegraphics[width=.95\linewidth]{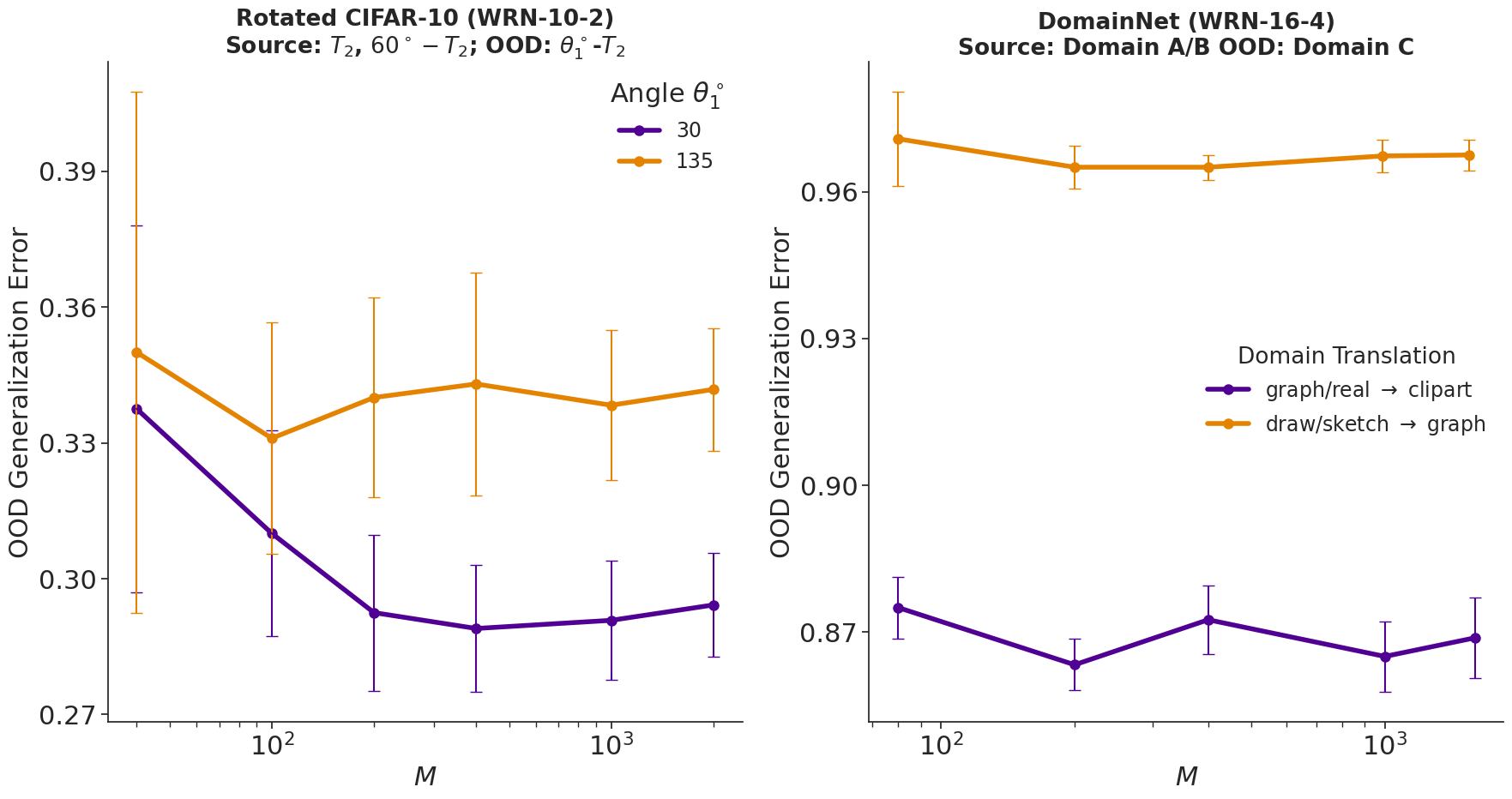}
    \caption{From two benchmark datasets, we plot their OOD generalization error ($y$-axis) as a function of the OOD sample sizes per class ($M$) ($x$-axis), namely
    \textbf{Left:} a classification task from Rotated CIFAR-10, where the OOD rotation is $\theta_1 = 30^{\circ}$ and $135^{\circ}$.
    \textbf{Right:} a classification task from DomainNet with OOD environment of \texttt{graph} and \texttt{clipart} respectively.
    We calculate the OOD generalization error over 10 runs and 3 seeds for the two datasets respectively.  We found a decrease at lower $M$ across all the pairs, and the average error is stable with a decreasing variance for larger values of $M$.
    Error bars indicate $95\%$ confidence intervals.
    }
    \label{fig:syn_m}
\end{figure}

\textbf{Generalization error decreases in power law despite shifts in dataset distributions.}
We take CINIC-10 as an example of distribution shifts between different datasets.
CINIC-10 is a dataset consisting of images selected from CIFAR-10 and down-sampled from ImageNet. We train a network on one subset of CINIC-10, use another subset as OOD samples, and test on it. \Cref{fig:syn_1}d shows that all situations exhibit a consistently decreasing trend, signifying that the OOD generalization error declines with an increase in the number of samples from the training dataset. Consequently, all models can perform well on another dataset without prior knowledge. The reason for this intriguing phenomenon is that although the datasets are separate (CIFAR-10 \vs ImageNet), the data distributions may not be significantly different.

\textbf{No effect of OOD sample size on generalization error.}
Unlike the above experiments where the number of training samples is fixed, we here investigate the impact of OOD samples on generalization error with two representative training environments in OOD datasets ($30^\circ$ and $150^\circ$ as OOD in Rotated CIFAR-10, \texttt{graph} and \texttt{clipart} as OOD in DomainNet).
As depicted in \Cref{fig:syn_m}, the generalization error declines with fewer target samples, however, as the number of OOD samples increases, the generalization error stays flat and is related to the degree of distribution shift. Thus, OOD generalization error is not associated with the number of samples in the unseen target domain, but rather with the training domains and the degree of shift in the target domain.

\textbf{Discussion.}
Even restricted to OOD benchmark datasets, different scenarios of generalization errors can be observed.
We have observed that the models can perform well on the interpolated distribution of training mixtures, that is, the error decreases as the sample size increases.
DNNs are highly non-convex, which makes it difficult to find the global optima, but they can generalize well on interpolated distributions.
The intuitive reason is that DNN performs well on the original distribution and as a continuous function, DNN with locally optimal values at the boundary of the original distributions may generalize equally well on the interpolated distributions.
Meanwhile, if the two distributions are similar enough, a well-trained DNN is also likely to perform well on the interpolated distributions because it has learned the common features between their original distributions.
When the unseen shift occurs, there is not always a downward trend. Models that perform well on the interpolated distribution fail when the test distribution is shifted significantly. In other words, the error may not decrease as the volume of training data grows.
In the next section, we revisit the OOD generalization problem and seek theoretical explanations for the non-descending phenomenon. 
\section{Revisit OOD generalization problem}

\label{sec:ood_def}
We first review those definitions and theories related to the OOD generalization problem.

\subsection{Formulation of the OOD generalization problem}
Consider a set of $N_{\mathcal{E}}$ environments (domains)  $\mathcal{E} = \{ e_i \}^{N_{\mathcal{E}}}_{i=1}$.
Let $\mathcal{E}_{s}$ and $\mathcal{E}_{t}$ be source and target environments collected from $\mathcal{E}$, respectively. That is, $\mathcal{E}_{s}\subset \mathcal{E}$ and $\mathcal{E}_{t}\subset \mathcal{E}$.  For each source environment $e\in\mathcal{E}_s$, there exists a training dataset $(X^e, Y^e) = \{(x^e_i, y^e_i)\}^{N_e}_{i=1}$ collected from each training environment. We use $x^e$ and $y^e$ to denote the generic sample and the label variables with respect to the environment $e$, respectively. Furthermore, denote the overall training dataset as $D_s =\{( X^e, Y^e): e \in  \mathcal{E}_{s}\}$.

We only have access to $\mathcal{E}_{s}$ during training, and the target environments are unseen during training and relatively different from the training one.
Moreover, we assume that there exists a ground-truth label process $h$ satisfying $h(x^e) = y^e$.
Then, in an OOD generalization, we would like to find a proper hypothesis function $f$
that minimizes the worst empirical risk among all the training environments $D_s$,
\begin{equation}
    \arg\min_{f} \sup_{e\in \mathcal{E}_s} R^e[\ell(h(x^e), f(x^e)],
\end{equation}
where $R^e$ denotes the ``empirical'' risk calculated over the loss  $\ell$ measuring the difference between the ground truth and the function $f$ for any sample $(x^e, y^e)$ in the training environments.
In this paper, we just consider a binary classification where the label $y^e$ is 0 or 1.

We then specify a set of assumptions about the data-generating environmental process and consider the OOD generalization error of interest. The described multi-environment model is general enough to cover both the i.i.d. case ($\mathcal{E}$ contains a single environment) and the OOD setup ($\geq 2$ environments are allowed) but also supports several other cases.
The difference among the environments is measured by $\mathcal{H}$-divergence as well as for domain adaption case~\citep{ben2010theory}.
\begin{equation}
    d_{\mathcal{H}} [e', e''] = 2 \sup_{f \in \mathcal{H}} | Pr_{x \sim e'}[\mathbf{I}(f)]-Pr_{x \sim e''}[\mathbf{I}(f)],
\end{equation}
where $\mathcal{H} = \{f: \mathcal{X}\mapsto \{0, 1\}\}$ is a hypothesis class on $\mathcal{X}$ and $\mathbf{I}(f) = \{x: f(x) = 1\}$, i.e., all the inputs $x\in\mathcal{X}$ that are classified as class 1 by the hypothesis $f$.

In the context of OOD generalization, the test distribution is inaccessible, necessitating certain assumptions about the test environment to enable generalization.

We address this issue in the results below and introduce generalization guarantees for a specific test environment with data mixture of training distributions.
Let the training environments $\mathcal{E}_s$ contain $N_{S}$ 
training environments, denoted as $\mathcal{E}_s = \{e^i_s\}^{N_S}_{i=1}$.   
The convex hull $Con(\mathcal{E}_s)$ of $\mathcal{E}_s$ is defined as the set of pooled environments given by
\begin{equation}
    Con(\mathcal{E}_s)=\{ \hat{e}\mid \hat{e} = \sum_{i=1}^{N_{S}} \alpha_i e^i_s,
    \alpha_i \in \Delta_{| N_{S}|-1}\},
\end{equation}
where $\Delta_{| N_{S} |-1}$ is the $(| N_{S} |-1)$-th dimensional simplex. The following lemma shows that for any pair of environments such that $e^{\prime}, e^{\prime \prime} \in Con(\mathcal{E}_{s})$, the  $\mathcal{H}$-divergence between  $e^{\prime}$ and  $e^{\prime \prime}$ is upper-bounded by the largest $\mathcal{H}$-divergence measured between the elements of $S$.

\begin{lemma}[Paraphrase from~\citep{albuquerque2019generalizing}] \label{lem:diver}
    Suppose $d_{\mathcal{H}}\left[e^i,e^j\right] \leq \epsilon, \forall i, j \in\left[N_{S}\right]$, then the following inequality holds for the $\mathcal{H}$-divergence between any pair of environments $e^{\prime}, e^{\prime \prime} \in Con(\mathcal{E}_{s})$:
    \begin{equation}
        d_{\mathcal{H}}\left[e^{\prime}, e^{\prime \prime}\right] \leq \epsilon.
    \end{equation}
\end{lemma}
It is suggestive that the $\mathcal{H}$-divergence between any two environments in $Con(\mathcal{E}_{s})$  (\ie, the maximum pairwise $\mathcal{H}$-divergence) can be used to measure the difference from the target environment and can affect OOD generalization ability of the model.

One generally refers to all unseen data as OOD data in literature work, provided such data is not available during training~\citep{ wang2022generalizing, zhou2022domain}. The classic OOD data is defined as follows
\begin{definition}[Out-of-distribution data (General)]
Let $\mathcal{E}_s$ and $\mathcal{E}_t$ be the sets of source and target environments, respectively, and let $\mathcal{Y}=[0, 1]$. Suppose we have a set of $n$ source environments, $\mathcal{E}_s = \{e_s^1, e_s^2, \ldots, e_s^{N_S}\}$.
Let $e_t \in \mathcal{E}$ be an unseen target environment, such that for any data in $e_t$, the following conditions hold
\begin{equation}
e_t \notin \{e_s^1, e_s^2, \ldots, e_s^{N_S}\},
\end{equation}
then the data in $e_t$ is said as \textbf{out-of-distribution data}.
\end{definition}

\subsection{Redefinition for OOD data}
It has been typically assumed in the literature that any test environment that does not appear in the training environments is considered OOD data. However, our findings suggest that OOD data exhibit heterogeneous generalization scenarios despite never having been encountered during training. We now use \Cref{lem:diver} to redefine whether the unseen data is OOD.
Formally,
\begin{definition}[Out-of-distribution data (Refined)]
\label{def:ood}
Let $\mathcal{E}_s$ and $\mathcal{E}_t$ be the sets of source and target environments, respectively, and $\mathcal{Y}$ be the output space. Suppose that we have a set of $n$ source environments, $ \mathcal{E}_s = \{e_s^1, e_s^2, \ldots, e_s^{N_S}\}$
and that there exists a real number $\epsilon > 0$ such that $d_{\widetilde{\mathcal{H}}} [e^i, e^j] \leq \epsilon,\; \forall e^i, e^j \in Con(\mathcal{E}_s)$.  Let $e_t \in \mathcal{E}$ be an unseen target environment
such that, for any data in $e_t$, the following conditions hold:
\begin{equation}
 d_{\widetilde{\mathcal{H}}}(e_t, \hat{e}) > \epsilon, \quad \forall \hat{e} \in Con(\mathcal{E}_s),
\end{equation}
then the data in $e_t$ is said as \textbf{out-of-distribution data},
where
$\hat{e} := \sum_{i=1}^{N_S} \alpha_i e^i_s$, where $\hat{e} \in \text{Con}(\mathcal{E}_s),\, \sum_{i=1}^{N_S} \alpha_i = 1$ is the convex combination of training environments, and $\alpha_i \geq 0$ for all $i$.
\end{definition}

\begin{remark}
    We discuss $\epsilon$ in this definition, noting that the mixture of training environments, $Con(\mathcal{E}_{s})$, can affect generalization to the target distribution. Threshold estimates are typically derived by assessing the model’s performance on a mixture of training data, which reflects the model’s ability on a specific distribution. However, unknown shifts can influence these estimates and are not directly measurable in practice. It is important to indicate that the data mixture in training may not cover all samples in the unseen target environment. Consequently, a model trained solely on this mixture may not be insufficient for generalizing to the target environment.
\end{remark}

\begin{remark}[An intuitive explanation of OOD data definition]
    The conventional intuitive way to understand OOD data is to treat it as inaccessible and not included in the training distribution. However, the various scenarios of generalization errors prompt us to introduce ``convex hull'' empirically, which helps define OOD data and the expected generalization. A classifier works via decision boundaries, which learns from the set of all mixtures obtained from given training distributions, \ie, convex hull. The distance between the target distribution and the decision boundary is a determinant of the classifier’s performance.  In the process of OOD generalization, unlike domain adaptation settings, no data from the test distribution can be observed. Furthermore, not all test samples are located outside the convex hull constructed by the training set, as defined above. For example, in \Cref{fig:eg},
    the model learning from \texttt{Painting} and \texttt{Photo} predicts better in  \texttt{Cartoon} than \texttt{Sketch} due to \texttt{Cartoon} is closer to the training mixture.
\end{remark}

Next, we present the error bounds in the following theorem.

\begin{theorem}[Upper-bounding the risk on unseen data]
\label{thm:error}
Let $\mathcal{E}_s$ be the set of training environments and let $\mathcal{Y}=[0, 1]$
For any unseen environment $e_t \in \mathcal{E}_{t}$ and any hypothesis $f\in \mathcal{H}$, the risk $R_t(f)$ can be bounded in the following ways:

(i) If $e_t \in Con(\mathcal{E}_s)$, data in $e_t$ is considered as ID, then
\begin{equation}
 \begin{split}
     R_t(f) &\leq \sum\nolimits_{i=1}^{N_S}\alpha_i  R_s^i(f) + 2\epsilon + \\
    &\min \{ \mathbb{E}_{\hat{e}} \|h_{s'} -h_t  \|, \mathbb{E}_{e_t} \|h_t -h_{s'} \| \},
 \end{split}
\end{equation}

(ii) If $e_t \notin Con(\mathcal{E}_s)$, data in $e_t$ is considered as OOD, then 
\begin{equation}
\begin{split}
     R_t(f) &\leq \sum\nolimits_{i=1}^{N_S} \alpha_i R_s^i(f) + \delta + \epsilon + \\
     & \min \{ \mathbb{E}_{\hat{e}} \|h_{s'} -h_t  \|, \mathbb{E}_{e_t} \|h_t -h_{s'}  \| \}.
\end{split}
\end{equation}
where  $\epsilon$ is the highest pairwise $\widetilde{\mathcal{H}}$-divergence measured between pairs of environments within $\text{Con}(\mathcal{E}_s)$ under $\widetilde{\mathcal{H}} = \{ sign(|f(x) - f'(x)| - t) \mid f, f' \in \mathcal{H}, 0\leq t \leq 1 \}$,
$ \delta := \min_{\alpha_i} d_{\widetilde{\mathcal{H}}} [e_{t}, \sum_{i=1}^{N_S} \alpha_i e_s^i]$ with minimizer $\alpha_i$ be the distance of $e_{t}$ from convex hull $Con(\mathcal{E}_s)$,
$\hat{e} := \sum_{i=1}^{N_S} \alpha_i e^i_s$ is the ``projection'' of $e_t$ onto convex hull $\text{Con}(\mathcal{E}_s)$ with $\alpha_i \geq 0$ for all $i$, $h_{s'}(x) = \sum_{i=1}^{N_S} \alpha_i h_{e_s^i}(x) $  is the labeling function for any $x \in Supp(\hat{e})$ derived from $\text{Con}(\mathcal{E}_s)$ with weights $\alpha_i$, and $h_t$ is the ground-truth labeling function for $e_t$.
\end{theorem}

\begin{proof}
    Let the source environment and target environments be $\mathcal{E}_s$ and  $\mathcal{E}_t$, respectively. The risk $R_t(h)$ can be bounded by~\citet{zhao19learning} for single-source and single-target domain adaptation cases as follows:
    \begin{equation} \label{eq:risk_t_proof}
    \begin{split}
        R_t(f) &\leq R_s(f) + d_{\widetilde{\mathcal{H}}} [e_{s}, e_{t}] + \\
        & \min \{ \mathbb{E}_{ e_s} \|h_s -h_t  \|, \mathbb{E}_{ e_t} \|h_t -h_s  \| \}.
    \end{split}
    \end{equation}
where $\widetilde{\mathcal{H}} = \{ sign(|f(x) - f'(x)| - t) \mid f, f' \in \mathcal{H}, 0\leq t \leq 1 \}$.
To design a generalized constraint for the risk of any unseen domain based on quantities associated with the distribution seen during training, we need to start by rewriting \Cref{eq:risk_t_proof} and considering $e_t$ and its ``projection" onto the convex hull of $\hat{e} \in Con(\mathcal{E}_s)=\{ \hat{e_s}| \hat{e_s} = \sum_{i=1}^{N_S} \alpha_i e_s^i, \sum_{i=1}^{N_S} \alpha_i = 1,\alpha_i \geq 0\}$.

For that, we introduce the labeling function $h_{s'}(x) = \sum_{i=1}^{N_S} \alpha_i h_{e_s^i}(x)$ which is an ensemble of the respective labeling functions and each weighted by the responding mixture coefficients from $Con(\mathcal{E}_s)$.
$ R_t(f)$ can thus be bounded as
\begin{equation}
    \begin{split}
         R_t(f) \leq & R_{\hat{e}}(f) +  d_{\widetilde{\mathcal{H}}} [\hat{e}, e_{t}] + \\
         &\min \{ \mathbb{E}_{\hat{e}} \|h_{s'} -h_t  \|, \mathbb{E}_{e_t} \|h_t -h_{s'}  \| \}.
    \end{split}
\end{equation}

Similarly to \Cref{lem:diver} for the case where $\mathcal{H} = \widetilde{\mathcal{H}}$, $e' = e_t$ and $e'' = \hat{e}$, it follows that
\begin{equation}
\begin{split}
        R_t(f) \leq & \sum\nolimits_{i=1}^{N_S} \alpha_i R_s^i(f) + \sum\nolimits_{i=1}^{N_S}\alpha_i d_{\widetilde{\mathcal{H}}} [e_s^i, e_{t}] + \\
         &\min \{ \mathbb{E}_{\hat{e}} \|h_{s'} -h_t  \|, \mathbb{E}_{e_t} \|h_t -h_{s'}  \| \}.
\end{split}
\end{equation}

Using the triangle inequality for the $\widetilde{\mathcal{H}}$-divergence along with \Cref{lem:diver}, we can bound the $\widetilde{\mathcal{H}}$-divergence between $e_t$ and any source environment $e_s^i$, $d_{\widetilde{\mathcal{H}}} [e_s^i, e_{t}]$, according to our new \Cref{def:ood}.

If $e_t \in Con(\mathcal{E}_s)$, the data in an environment $e_t$ is defined as in-distribution data, which means even $e_t$ has never been seen at training, it is still located in the convex hull of training mixture.
According to \Cref{lem:diver} for case I where $e' = e_t$ and $e'' = \hat{e}$, and case II where $e' = e_s^i$ and $e'' = \hat{e}$, the following inequality holds for the $\mathcal{H}$-divergence between any pair of environments $e',e'' \in Con(\mathcal{E}_s)$.
 \begin{equation}
 \begin{split}
    d_{\widetilde{\mathcal{H}}} [e_s^i, e_{t}] &\leq d_{\widetilde{\mathcal{H}}} [e_s^i,\hat{e}]
 + d_{\widetilde{\mathcal{H}}} [\hat{e}, e_{t}] \\
 & \leq  \epsilon + \epsilon  = 2\epsilon.
 \end{split}
\end{equation}

And if $e_t \notin Con(\mathcal{E}_s)$, the data in an environment $e_t$ is defined as out-of-distribution data, and we have
 \begin{equation}
 \begin{split}
    d_{\widetilde{\mathcal{H}}} [e_s^i, e_{t}] &\leq d_{\widetilde{\mathcal{H}}} [e_s^i,\hat{e}]
 + d_{\widetilde{\mathcal{H}}} [\hat{e}, e_{t}] \\
 & \leq \epsilon + \delta ,
 \end{split}
\end{equation}
where $ \delta := \min_{\alpha_i} d_{\widetilde{\mathcal{H}}} [e_{t}, \sum_{i=1}^{N_S} \alpha_i e_s^i]$ with minimizer $\alpha_i$ be the distance of $e_{t}$ from convex hull $Con(\mathcal{E}_s)$.
Thus, we get an upper-bounded $d_{\widetilde{\mathcal{H}}} [e_s^i, e_{t}]$ based on our new definition.

Finally we rewrite the bound  on $R_t(f)$:
 if $e_t \in Con(\mathcal{E}_s)$, the data in an environment $e_t$ is defined as in-distribution data
\begin{equation}
 \begin{split}
     R_t(f) &\leq \sum\nolimits_{i=1}^{N_S} \alpha_i R_s^i(f) + 2\epsilon + \\
    &\min \{ \mathbb{E}_{\hat{e}} \|h_{s'} -h_t  \|, \mathbb{E}_{e_t} \|h_t -h_{s'} \| \}.
 \end{split}
\end{equation}
 And if $e_t \notin Con(\mathcal{E}_s)$, the data in an environment $e_t$ is defined as out-of-distribution data
\begin{equation}
\begin{split}
     R_t(f) &\leq \sum\nolimits_{i=1}^{N_S} \alpha_i R_s^i(f) + \delta + \epsilon + \\
     & \min \{ \mathbb{E}_{\hat{e}} \|h_{s'} -h_t  \|, \mathbb{E}_{e_t} \|h_t -h_{s'}  \| \}.
\end{split}
\end{equation}

\end{proof}
\begin{remark}[Intuitive interpretation of \Cref{thm:error}]
    Compared with the upper bound proposed by~\citet{albuquerque2019generalizing}, our theorem derives different upper bounds for unseen data in different cases based on the new definition. For unseen data within the convex hull of training mixture, the upper bound is tighter because $e_t$ is assumed to be close enough to the training mixture. The upper bound is more relaxed for OOD data in $e_t$ because additional differences between the target environment and the source environment need to be taken into account (denoted by $\delta$).
    Such boundaries explain why data mixture in training cannot guarantee models' OOD capability. The intuitive interpretation of this theorem can also be verified in \Cref{fig:eg}.
    Different risk bounds on unseen data mean that in some cases, the performance of the model may be affected by different bounds even if a large amount of training data is provided, especially if the multi-domain data collected in the real world has different data quality, task difficulty, and distribution gap.
    This all means that even with more training samples, the OOD generalization ability of the model may be limited. In other words, the model will not improve indefinitely even if more training data is added.
\end{remark}

\begin{remark}[The importance of diverse data]
    We further highlight that the introduced results also provide insights into the value of obtaining diverse datasets for generalization to OOD samples in practice. More diverse the environments where the dataset occurs during training, it is more likely that the unseen distribution falls within the convex hull of the training environments. That is, even though the dataset has never been seen before, it may still belong to ID data.
    Specifically, the diversity of training samples can help the model learn more robust and general feature representations that are critical for identifying and processing new, unseen data distributions.
    Thus, as well as the dataset size, the diversity of training samples is also crucial for better generalization.
\end{remark}
\begin{remark}[Widely used OOD techniques]
    We next discuss widely used OOD techniques based on the introduced framework in order to demonstrate the verification of our theory and interesting observations in related algorithms.
    Such techniques are crucial for enhancing the robustness and OOD generalization of DNNs, including data augmentation, pre-training, and hyperparameter optimization. We then define a novel selection algorithm that relies only on training data. While the risk of utilizing the training mixture can be minimized, it is able to measure the value of data and adapt the learning scope dynamically. It is worth noting that our empirical result is that the proposed algorithm can achieve success even when no information is known about the target environment.
\end{remark}

\section{Can we break OOD limitation to improve model's capability?}
\label{sec:method}
\emph{Can we make OOD samples as generalizable as possible}, so that we can improve the generalization ability of the models and we can provide reasonable explanations when we do not have any prior knowledge about the target distribution?

\subsection{Observation for widely used OOD methods}
\label{sec:method_pop}
When given unknown samples that are never seen during training, the number of available options to alleviate the degradation caused by OOD samples is limited.
Thus we need to use some popular techniques to process unseen data to make them as ``close'' as possible to the data mixture in training, such as hyperparameter optimization, data augmentation, pre-training, and other DomainBed algorithms.

\begin{figure}[!tb]
    \centering
    \includegraphics[width=.9\linewidth]{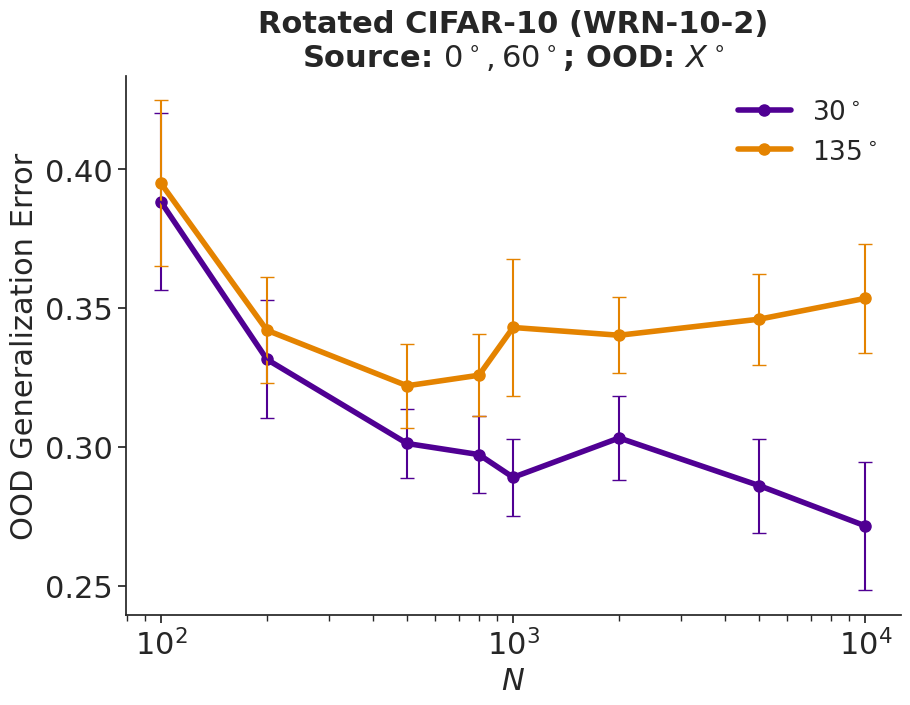}
    \caption{For $0^{\circ}$ and $60^{\circ}$ as source samples, and $135^{\circ}$ and $30^{\circ}$ as OOD samples in Rotated CIFAR-10 sub-task $T_2$ respectively, we investigate the effect of \textbf{hyper-parameter tuning}. We record the best set of hyper-parameters with a validation set and test it on an unseen target. It can still be observed that the same error trend in our previous results since manipulating the training set is irrelevant for the test set, and the distribution distance is the main influencing factor.}
    \label{fig:optim}
\end{figure}

\textbf{Effect of hyperparameter optimization.}
The first technique we aim to question is whether the best-performing model trained on the training set can effectively reduce the generalization error on an unseen distribution. Similarly as by ~\citet{kumar2022fine}, we employ an easy two-step strategy of linear probing and then full fine-tuning (50-50 epochs).  It can be observed in \Cref{fig:optim} that relying on hyperparameter optimization alone is not sufficient to improve model performance in the case of handling OOD samples.
This implies that the optimal performance of a model on training or validation data does not guarantee its success on unseen samples. The reason for this is that it adjusts primarily for the nature of training data, but does not offer enough insight regarding the OOD samples.

\begin{figure}[tb]
    \centering
    \includegraphics[width=.95\linewidth]{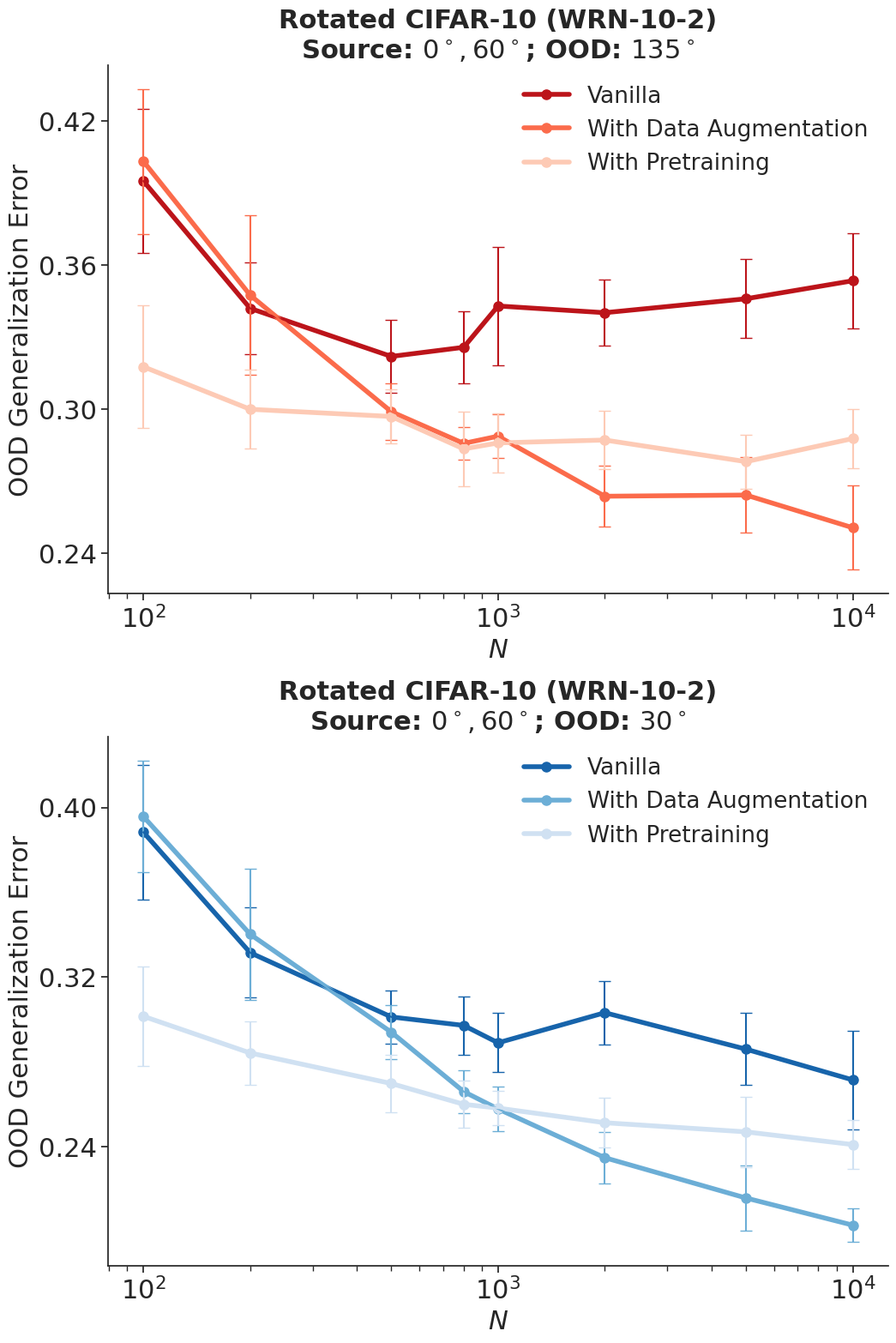}
    \caption{For $0^{\circ}$ and $60^{\circ}$ as training samples, and $135^{\circ}$(Upper) and $30^{\circ}$(Bottom) as OOD samples in Rotated CIFAR-10 sub-task $T_2$ respectively, we train a WRN-10-2 model with target samples $M=400$, under the following settings:
    (1) \emph{Vanilla}, that is, without any popular techniques (darkest color),
    (2) \emph{Data augmentation} with random copping, flips and padding (medium color), and
    (3) \emph{Pre-training} followed by fine-tuning (lightest color). WRN-10-2 is pre-trained on ImageNet images from CINIC-10 (100 epochs and 0.01 learning rate followed~\citet{de2023value}).
    For smaller $N$, augmentation worsens the OOD generalization error but shows improvement with increasing samples, not disturbed by rotation. After pre-training, WRN-10-2 initially has a dramatic drop in error but still rises for unseen samples, especially on $135^{\circ}$.
    Error bars indicate $95\%$ confidence intervals (10 runs).}
    \label{fig:da_pre}
\end{figure}

\textbf{Effect of data augmentation.}
To evaluate whether augmentation works, we use two different rotations as unseen tasks for WRN-10-2 on Rotated CIFAR-10, \ie, $135^{\circ}$ and $30^{\circ}$ as OOD tasks.
The results in \Cref{fig:da_pre} (medium color) show that the effectiveness of data augmentation increases as the amount of data increases. Initially, for a small dataset, augmentation may have a negative effect. However, as the size of training data grows, augmentation helps break the bias (digit \vs rotation) in the training data.
Augmentation overcomes the overfitting phenomenon for semantic noise, especially on $135^\circ$ as an OOD task.
Its effectiveness stems from the ability to introduce diversity into the training data patterns.
Data augmentation helps create augmented samples that better represent real-world variations and challenges. This enables the model to learn more robust and generalizable features, improving its performance during tests on OOD samples.

\textbf{Effect of pre-training.}
We repeat the same target tasks for pre-training on Rotated CIFAR-10.
As shown in \Cref{fig:da_pre} (lightest color), even with a limited sample size, pre-training demonstrated a more robust improvement.
And as the sample size increases, the improvement of pre-training becomes more minor, but still better than the baseline.
We can conclude that pre-training is a useful tool for improving the ability to unseen distributions.
Pre-training on large and diverse datasets can enhance the model's ability to learn a broad and generalized set of representations, which can subsequently improve its performance on OOD samples.
However, its effectiveness depends on how well they can transfer the learned representations to target tasks. The quality of the representations must have been tested carefully based on target tasks.

\begin{figure}
    \centering
    \includegraphics[width=.9\linewidth]{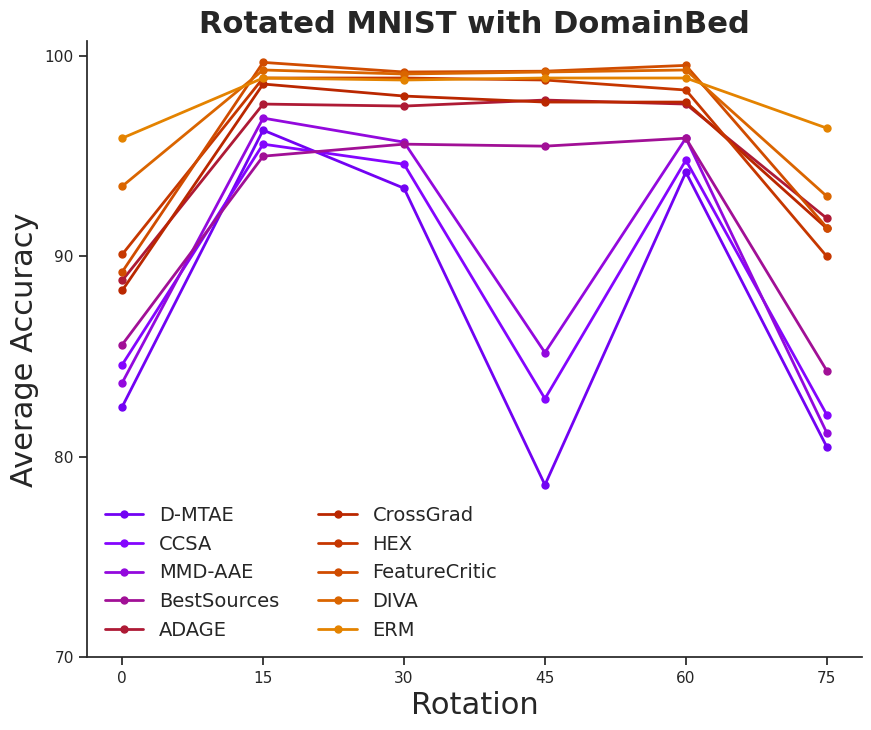}
    \caption{Five of the six domains ($\{ 0^{\circ}, 15^{\circ}, 30^{\circ}, 45^{\circ}, 60^{\circ}, 75^{\circ}\}$) in Rotated MNIST  are used as training sets and the rest as an unseen task.
    We validate the performance data of 10 algorithms in DomainBed on unseen rotations. The higher the OOD generalization accuracy, the better the algorithm in DomainBed is at handling unseen targets. It can be observed that all algorithms perform poorly for $0^{\circ}$ and $75^{\circ}$, demonstrating sensitivity to unseen shifts.}
    \label{fig:domainbed}
\end{figure}

\textbf{The sensitivity of distribution distance in DomainBed algorithms.}
We can see the performance of different algorithms on the rotated MNIST dataset in \Cref{fig:domainbed}.
Most algorithms (like ERM and DIVA~\citep{ilse2020diva}) perform well at rotations of $15^{\circ}-60^{\circ}$, but mediocre at $0^{\circ}$ and $75^{\circ}$. This trend verifies that different algorithms are equally sensitive to changes in data distribution. That is, they perform well in the environment within the data mixture but perform poorly in the environment with distant distribution.  Besides, the performance of some algorithms fluctuates greatly from different angles. For example, MMD-AAE~\citep{li2018domain} and BestSources~\citep{8451318} perform well at lower angles, but their performance degrades more after $45^{\circ}$.
This may be due to the different feature extraction processes of the algorithms, resulting in generalization failure at $45^{\circ}$. The performance of existing algorithms is sensitive to the distribution distance, which can help us understand the generalization ability of different algorithms in OOD settings.

\subsection{Selection of the training samples}
\label{sec:method_ours}
Recent work suggests that a careful selection of the most valuable samples can prevent models from focusing too much on the noise in training, thereby having a potential reduction of the overfitting risk~\citep{ngiam2018domain, yoon2020data}.
However, data selection for the OOD generalization problem meets a train-test mismatch challenge, since it is impossible to anticipate the distribution of target data that the model will encounter in the future.

\subsubsection{Diversity learning for OOD generalization}
Diversity is a desirable dimension in OOD generalization problems.
We can explore dynamically adjusting weight and diversity centered on the training data, so that information from the source can be \emph{``partially adapted''} to the target.
For example, style diversity in the dog data, since each style (\texttt{picture}, \texttt{sketch}, \texttt{painting}, \etc) is different.
Our empirical result suggests that diversity ensures the expansion of the convex hull and contains more unseen samples as ID.
To increase training diversity, we can use random weighting, since it can help to increase the amount of domain-agnostic information available and promote the robustness of neural networks.
Although random selection can capture some of the information in training domains, it depends on the distance between domains and cannot correspond to each training domain sample one-to-one, and therefore cannot be updated using traditional stochastic gradient descent. Fortunately, well-established solutions from the field of reinforcement learning (RL) are readily available to update the selection sampler and weight each sample individually to the selection criteria~\citep{yu2022can, yoon2020data}.

\begin{figure}[!tb]
    \centering
    \includegraphics[width=\linewidth]{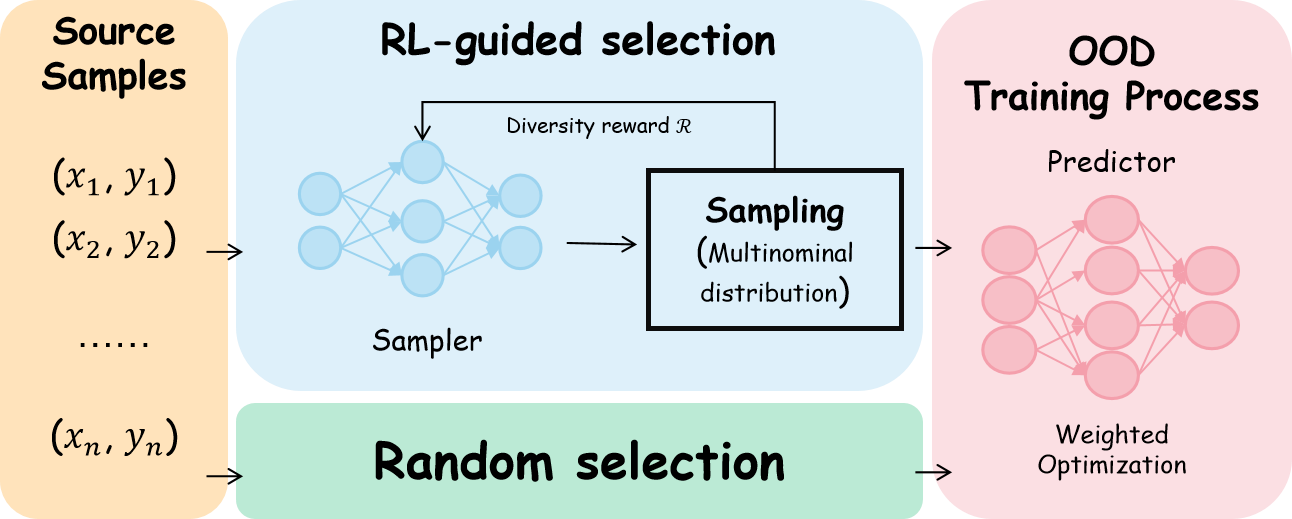}
    \caption{Block diagram of Selecting the training samples
    A set of training domain samples serves as training input.
    (1) \emph{RL-guided selection:} Learns from training samples (with shared parameters between batches), updates with diversity-related rewards, and returns selection vectors (corresponding to a multinomial distribution).
    (2) \emph{Random selection:} Randomly output select vectors to a batch of training samples.
    The OOD predictor is trained only on training samples with selection vectors, using gradient descent optimization.}  \label{fig:sampler}
\end{figure}

\subsubsection{Framework Details}
We demonstrate that our selection performance in two OOD tasks: (1) $0^\circ$ and $60^\circ$ as training environments and $135^\circ$ as OOD samples in Rotated CIFAR-10; (2) \texttt{clipart} and \texttt{sketch} as training environments and \texttt{painting} as OOD samples in PACS.
We formalize the components of the sample selection in OOD training set optimization as shown in \Cref{fig:sampler}, including a training dataset $D_s$ with size $n$, a predictor $f$ for OOD generalization and an encoder $f_{en}$ as a feature extractor.
We first shuffle and randomly partition $D_s$ into mini-batches, and then perform the data selection process.
(i) As for \emph{random selection}, we provide a random vector for each mini-batch selection.
(ii) As for \emph{RL-guided selection}, we adopt REINFORCE~\citep{williams1992simple} algorithm to train a selector $\mathcal{F}$ for the optimization of OOD generalization. Our goal is to learn an optimal policy $\pi$ to maximize the diversity of selected subsets from each mini-batch.

The process of RL-guided Selection is:
\emph{First,} the encoder $f_{en}$ transforms a batch of data $B_t$ into its representation vector $v_t (v_t = f_{en}(B_t))$ at each step $t$.
\emph{Secondly,} the policy $\pi$ outputs the batch of state $s_t$, so that each $v_t$ is associated with a probability of diversity representing how likely it is going to be selected. The selected subset $\hat{B_t}$ is then obtained by ranking their probability.
\emph{Thirdly,} the selector $\mathcal{F}$ as well as encoder $f_{en}$ are finetuned by the selected subset $\hat{B_t}$. The $\mathcal{F}$ is updated with REINFORCE algorithm and the reward function $\mathcal{R}$. No target knowledge is required, and the reward is only measured via max divergence of source samples. We define the divergence $Diver(B)$ as the total of all distances between any pairs $(v_i, v_j)$ within the batch $B$:
\begin{equation}\label{eq:divergence}
    Diver(B) = \sum d(v_i, v_j), (v_i, v_j)\in B
\end{equation}
where $d(\cdot, \cdot)$ is the distance function.
Maximizing the dispersion of samples can effectively increase the coverage of learned representations of well-trained models, leading to a more diverse content for the target environment. See Algorithm \ref{alg:rl_method} for more details.
The expansion of convex hull allows models to learn more diverse types of features and patterns.

\begin{algorithm}[!tb] \small
    \caption{RL-guided data selection}
    \label{alg:rl_method}
    \KwIn{Training set $D_s$, Predictor $f$ (including encoder $f_{en}$), Epoch $M$, Reward function $\mathcal{R}$, learning rate $\alpha$, Discount factor $\gamma$}
    \KwOut{selected set, fine-tuned predictor $f$, policy $\pi$, data selector $\mathcal{F}$}
    Initialize selection policy $\pi$ and data selector $\mathcal{F}$\;
    \ForEach{episode $m$}{
    Shuffle and randomly partition $D_s$ into mini-batched with same size $N$: $D_s = \{B_t \}_{t=1}^{T} = \{ B_1, B_2, \dots, B_T\}$ \;
    Initialize an empty list for episode history $\Gamma$\;
    \ForEach{$B_t \in D_s$}{
    $v_t = f_{en}(B_t)$\;
    Obtain state $s_t$\;
    Obtain action $a_t$ by sampling based on $\pi(s_t)$ \;
    Obtain the selected set $\hat{B_t}$ by ranking  $a_t$\;
    Fine-tune predictor $f$ with $\hat{B_t}$\;
    Calculate reward $r_t = \mathcal{R}(\hat{B_t}, \mathcal{F})$ with \Cref{eq:divergence}\;
    Store $(s_t, a_t, r_t)$ to episode history $\Gamma$\;
    }
    \ForEach{$(s_t, a_t, r_t) \in \Gamma$}{
    Update policy weight and selector weights with REINFORCE algorithm
    with $\mathcal{R}$,  $\alpha$, and $\gamma$.}
    Clear episode history $\Gamma$\;
    }
    return $f$, $\pi$, and $\mathcal{F}$
\end{algorithm}

\begin{figure*}[tb]
    \centering
    \subfloat[Generalization error on Rotated CIFAR-10 and PACS]{\includegraphics[width=0.6\linewidth]{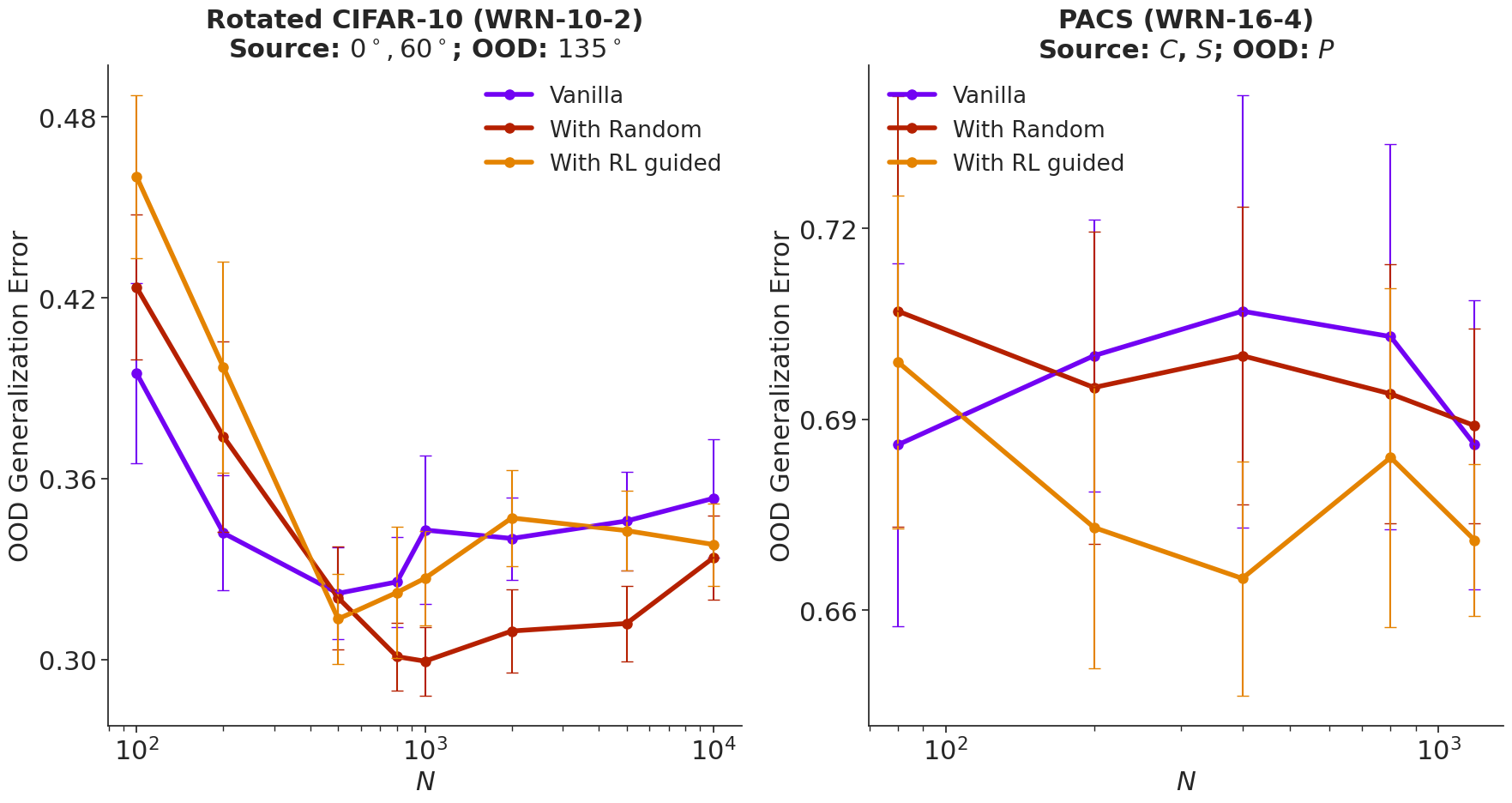} \label{fig:rl}}
    \subfloat[Training reward on PACS]{\includegraphics[width=0.35\linewidth]{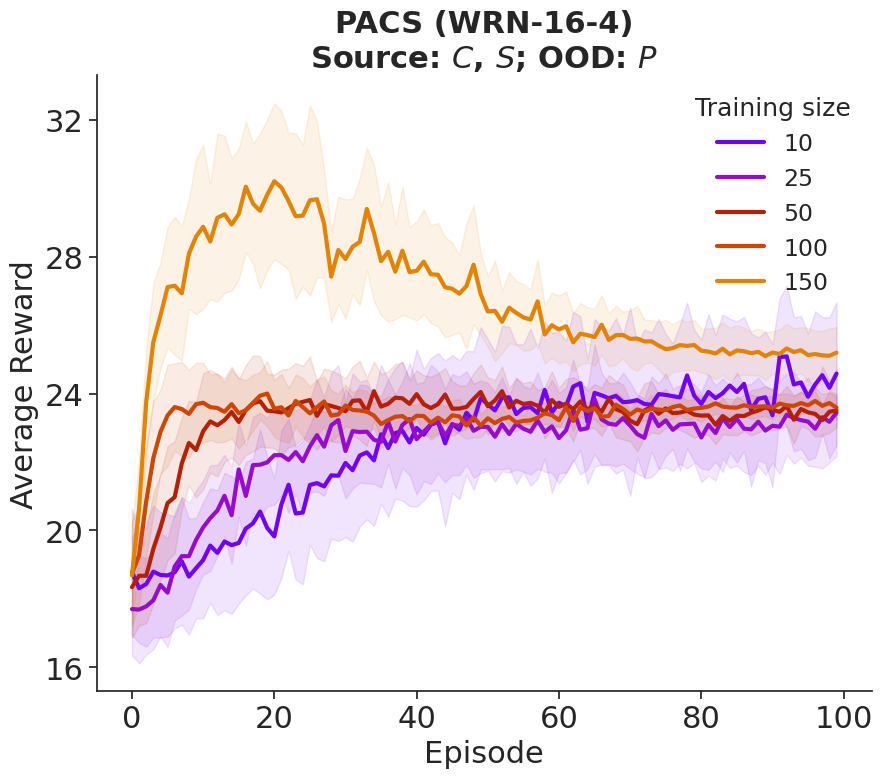}\label{fig:rl-reward}}
    \caption{The selecting results for OOD training.
    \textbf{(a)} Generalization error on the target distribution for Rotated CIFAR-10 (left) and PACS (right) using weighted subset. Here, we present three settings on OOD tasks: \textbf{Vanilla}, \textbf{With Random}, and \textbf{With RL-guided weighted objective}. All settings only get information from the training environment. For a small number of samples, the error of the latter two is high, but as the sample size grows, they can effectively lower the OOD generalization error. 
    The effect of OOD data arising due to intra-class nuisances (Rotated CIFAR-10) is more sensitive to random weights since there is no semantic-level shift and more correlation-level.
    Unlike in CIFAR-10 tasks, we observe that in PACS, OOD generalization error falls significantly due to semantic-level shift. In other words, if we use a weighted training subset, then we always obtain some benefit on OOD samples, whether random or guided. 
    \textbf{(b)} Training reward on PACS of RL-guided selection method. Reward on all training sizes converge to the optimal reward associated with PACS dataset.
    Error bars indicate $95\%$ confidence intervals (10 runs).}
    
\end{figure*}

\subsubsection{Results and discussion}
We first provide the training reward on PACS in \Cref{fig:rl-reward}.
We can observe that the proposed algorithm gradually reaches a convergence point at approximately $40$ episodes, which exhibits a faster convergence speed.
If the training size is large (\eg, $150$), the reward first increases and then decreases until it flattens out. Additionally, the variance of average reward changes during the increase in training size. It can be concluded that by designing diversity as the reward function, the neural network can learn more diverse examples
Moreover, in the same tasks presented in \Cref{sec:obser_real}, no noticeable trends were indicating a decrease in performance. However, a significant reduction in error can be observed in \Cref{fig:rl} following the selection process.

Weighted objectives were once believed to be ineffective for over-parameterized models, including DNN~\citep{byrd2019effect} due to zero-valued optimization results (\ie, the model fits the training dataset perfectly, resulting in zero loss). However, our experiments demonstrate the opposite: our method outperforms the vanilla in the case of large training samples, whether random or elaborated weighted. This is perhaps because cross-entropy loss is hard to reach 0 due to multi-distribution distance. Simultaneously, the selection of samples exhibiting the largest dissimilarities between distributions has the potential to broaden the coverage of the convex hull in the training domains. This expansion ultimately leads to improved performance outcomes.

By utilizing weighted objectives, we effectively prioritize and emphasize samples that significantly contribute to the overall learning process. This approach allows the model to focus on the most informative and challenging parts of the training mixture, thereby enhancing its capacity to generalize to new, unseen shifts. Furthermore, our selection method operates independently of the necessity for environment labels (\ie, identification of sample domains), making it more flexible and applicable in a broader range of scenarios. It can be seamlessly integrated with other OOD generalization techniques, such as domain generalization or data augmentation, to further enhance the model's ability to generalize to new environments.
\section{Related Work}
\label{sec:related_work}
\textbf{OOD generalization.}
Theoretical achievements in machine learning has been made all under the assumption of independent and identical distributions (\iid assumption).
However, in many real-world scenarios, it is difficult for the \iid assumption to be satisfied, especially in the areas such as medical care~\citep{schrouff2022diagnosing}. Consequently, the ability to generalize under distribution shift gains more importance.
Early OOD studies mainly follow the distribution alignment by learning domain invariant representations via kernel methods~\citep{hu2020domain}, or invariant risk minimization~\citep{chen2022pareto}, or disentangle learning~\citep{zhang2022towards}.
Research on generalizing to OOD data has been extensively investigated in previous literature, mainly focusing on data augmentation~\citep{yao2022Improving} or style transfer~\citep{zhang2022Exact}.
Increasing data quantity and diversity from various domains enhances the model's capability to handle unseen or novel data.
\citet{zhang2017mixup} introduced Mixup, which generates new training examples by linearly interpolating between pairs of original samples. \citet{zhou2020Domain} further extended this idea by Mixstyle, a method that leverages domain knowledge to generate augmented samples. Other OOD methods on the level also employ cross-gradient training~\citep{shankar2018Generalizing} and Fourier transform~\citep{xu2021FourierBased}.

Different from all the methods mentioned above, our work starts with the definition of OOD data and its empirical phenomenon and revisits the generalization problem from a theoretical perspective.
It is worth noting that the difference between us and \citet{de2023value} is that the latter focuses on the effects of adding a few OOD data to the training data, while we focus on the effects of generalizing to unseen target tasks.

\textbf{Data selection.}
Data selection is a critical component of the neural network learning process, with various important pieces of work~\citep{sorscher2022beyond}.
Careful selection of relevant and representative data is to guarantee that the data used for training accurately captures the patterns and relationships that the network is supposed to learn.
Several recent studies have explored different metrics for quantifying individual differences between data points, such as EL2N scores~\citep{yoon2020data} and forgetting scores~\citep{toneva2018empirical}.
There are also influential works on data selection that contribute to OOD and large pre-trained models.
\citet{zhu2023xtab} introduced a framework of cross-table pre-training of tabular transformers on datasets from various domains.
\citet{shao2023synthetic} leveraged manually created examples to guide large language models in generating more effective instances automatically and select effective ones to promote better inference.

While these methods can improve the training data quality, our method leverages existing data with reinforcement learning in source domains for modeling to maximize diversity.
Discussions concerning the convex hull and OOD data selection have also been prevalent in literature~\citep{zhou2022model, ye2022ood}. For instance, \citet{krueger2021out} aim to optimize the worst-case performance of the convex hull of training mixture. Our work distinguishes itself by focusing on error scenarios resulting from distribution shifts on DNN and highlighting the importance of data diversity. Additionally, we provide novel insights for algorithm design. 
\section{Conclusion}
This work examined the phenomenon of non-decreasing generalization error when the models are trained on data mixture of source environments and the evaluation is conducted on unseen target samples. Through empirical analysis on benchmark datasets with DNN, we introduced a novel theorem framework within the context of OOD generalization to explain the non-decreasing trends. Furthermore, we demonstrated the effectiveness of the proposed theoretical framework in the interpretation of the existing methods by evaluating existing techniques such as data augmentation and pre-training. We also employ a novel data selection algorithm only that is sufficient to deliver superior performance over the baseline methods.

\section*{Acknowledgment}
This work is supported by .
\ifCLASSOPTIONcaptionsoff
  \newpage
\fi



\bibliographystyle{IEEEtranN}
\bibliography{main}

%



\end{document}